\documentclass[letterpaper]{article}

\PassOptionsToPackage{numbers, sort, compress}{natbib}

    \usepackage[preprint]{neurips_2020}

\usepackage[utf8]{inputenc} 
\usepackage[T1]{fontenc}    
\usepackage{hyperref}       
\usepackage{url}            
\usepackage{booktabs}       
\usepackage{amsfonts}       
\usepackage{nicefrac}       
\usepackage{microtype}      

\usepackage{bm}
\usepackage{tikz}
\usepackage{xfrac}
\usepackage{amsthm}
\usepackage{amsmath}
\usepackage{amssymb}
\usepackage{multicol}
\usepackage{multirow}
\usepackage{pgfplots}
\usepackage{subcaption}

\usepackage[inline]{enumitem}

\usepackage{algorithm}
\usepackage{algorithmic}

\usetikzlibrary{arrows}

\usepackage{tcolorbox}
\usepackage{todonotes}

\newtheorem{innercustomgeneric}{\customgenericname}
\providecommand{\customgenericname}{}
\newcommand{\newcustomtheorem}[2]{%
    \newenvironment{#1}[1]
    {%
        \renewcommand\customgenericname{#2}%
        \renewcommand\theinnercustomgeneric{##1}%
        \innercustomgeneric
    }
    {\endinnercustomgeneric}
}

\newcustomtheorem{customthm}{Theorem}
\newcustomtheorem{customlemma}{Lemma}

\newtheorem{definition}{Definition}

\newtheorem{lemma}{Lemma}
\newtheorem{theorem}{Theorem}

\DeclareMathOperator{\defeq}{\triangleq}

\DeclareMathOperator{\E}{\mathbb{E}}
\DeclareMathOperator{\M}{\mathbb{M}}

\DeclareMathOperator{\T}{\boldsymbol{\theta}}

\title{A Natural Actor-Critic Algorithm with\\Downside Risk Constraints}

\author{%
    Thomas Spooner {\normalfont and} Rahul Savani\\
    Department of Computer Science\\
    University of Liverpool\\
    \texttt{\{t.spooner, rahul.savani\}@liverpool.ac.uk}
}

\begin{document}

\maketitle

\begin{abstract}

    Existing work on risk-sensitive reinforcement learning --- both for
    symmetric and downside risk measures --- has typically used direct
    Monte-Carlo estimation of policy gradients. While this approach yields
    unbiased gradient estimates, it also suffers from high variance and
    decreased sample efficiency compared to temporal-difference methods.
    In this paper, we study prediction and control with aversion to downside
    risk which we gauge by the lower partial moment of the return. We introduce
    a new Bellman equation that upper bounds the lower partial moment,
    circumventing its non-linearity. We prove that this proxy for the lower
    partial moment is a contraction, and provide intuition into the stability
    of the algorithm by variance decomposition. This allows sample-efficient,
    on-line estimation of partial moments.
    For risk-sensitive control, we instantiate Reward Constrained Policy
    Optimization, a recent actor-critic method for finding constrained
    policies, with our proxy for the lower partial moment. We extend the method
    to use natural policy gradients and demonstrate the effectiveness of our
    approach on three benchmark problems for risk-sensitive reinforcement
    learning.

\end{abstract}

\section{Introduction}
Reinforcement learning (RL) solves the problem of how to act optimally in a
potentially unknown environment. While it does this very well in many cases, it
has become increasingly clear that uncertainty about the environment --- both
epistemic and aleatoric in nature --- can have severe consequences on the
performance of our algorithms.  While many problems can be solved by maximising
the expected returns alone, it is rarely sufficient, and shies away from many
of the subtleties of the real-world. In fields such as finance and health, the
mitigation of risk is absolutely foundational, and the lack of practical
methods is one of the biggest roadblocks in wider adoption of RL. Now, recent
developments in risk-sensitive RL have started to enable practitioners to
design algorithms to tackle their problems. However, many of these approaches
rely on full trajectory rollouts, and most only consider variance-related risk
criteria which:
\begin{enumerate*}
    \item are not suited to all domains; and
    \item are often non-trivial to estimate in an on-line setting.
\end{enumerate*}
We rarely have the luxury of ready access to high-quality data, and our
definition of risk is usually
nuanced~\cite{tversky1979prospect,shafir2002rationality}.

This observation is not unique and indeed many fields have questioned the use
of symmetric risk measures to correctly capture human preferences. Markowitz
himself noted, for example, that ``semi-variance seems a more plausible measure
of risk than variance, since it is only concerned with adverse
deviations''~\cite{markowitz1991foundations}. Yet, save
for~\citet{tamar2015policy} --- who introduce semi-deviation as a possible
measure of risk --- very little work has been done to address this gap in RL
research. Furthermore, of those that do, even fewer still consider the question
of how to learn an incremental approximation, instead opting to directly
estimate policy gradients with sampling.

The \textbf{first contribution} of this work lies in the development of the
lower partial moment (LPM) --- i.e.\ the expected value of observations falling
below some threshold --- as an effective downside risk measure \emph{that can
be approximated efficiently through temporal-difference learning}. This insight
derives from the sub-additivity of the \(\max\) function and enables us to
define a recursive bound on the LPM that serves as a proxy in constrained
policy optimisation. We are able to prove that the associated Bellman operator
is a contraction, and analyse the variance on the transformed reward that
emerges from the approximation to gain insight into the stability of the
proposed algorithm. The \textbf{second key contribution} is to show that the
Reward Constrained Policy Optimisation framework (\texttt{RCPO})
of~\citet{tessler2019reward} can be extended to use natural policy gradients.
While multi-objective problems in RL are notoriously hard to
solve~\cite{mannor2013algorithmic}, natural gradients are known to address some
of the issues associated with convergence to local minima. The resulting
algorithm used alongside our LPM estimation procedure is easy to implement and
is shown to be highly effective in a number of problem settings.

\paragraph{Related work.}
Past work on risk-sensitivity and robustness in RL can be split into those that
tackle epistemic uncertainty, and those that tackle aleatoric uncertainty --
which is the focus of this paper. Aleatoric risk (the risk inherent to a
problem) has received much attention in the literature. For example,
in~\citeyear{moody2001learning}, \citeauthor{moody2001learning} devised an
incremental formulation of the Sharpe ratio for on-line learning.
\citet{shen2014risk} later designed a host of value-based methods using utility
functions (see also~\cite{lefebvre2017behavioural}), and work
by~\citet{tamar2016learning} and \citet{sherstan2018directly} even tackle the
estimation of the variance on returns; a contribution closely related to those
in this paper. More recently, a large body of work that uses policy gradient
methods for risk-sensitive RL has emerged
\cite{tamar2012policy,tessler2019reward,bisi2019risk,tamar2015policy}.
Epistemic risk (the risk associated with, e.g., known model inconsistencies)
has also been addressed, though to a lesser
extent~\cite{pinto2017robust,klima2019robust,spooner2020robust}. There also
exists a distinct but closely related field called ``safe RL'' which includes
approaches for safe exploration; see the excellent survey
by~\citet{garcia2015comprehensive}.

\section{Preliminaries}\label{sec:prelims}
\subsection{Markov decision processes}
A regular discrete-time Markov decision process (MDP) comprises: a state space
$\mathcal{S}$, (state-dependent) action space $\mathcal{A}(s) \subseteq
\mathcal{A}$, and set of rewards $\mathcal{R} \subseteq \mathbb{R}$. The
dynamics of the MDP are given by the state-transition probability distribution
$p(s' \mid s, a)$ with initial state distribution $d_0(s)$. The expected value
of rewards generated for a state-action pair is denoted by $r(s, a)$. A policy
$\pi_{\T}(a \mid s)$, parameterised by the length \(n\) vector \(\T \in
\mathbb{R}^n\), assigns a probability density over the set of possible actions
in a state $s$. We assume that \(\pi\) is continuously differentiable with
respect to \(\T\). For a given policy we define the return starting from time
\(t\) (in unified notation) by the sum of future rewards, \(G_t \defeq
\sum_{k=0}^T \gamma^{k} R_{t+k+1}\), where \(\gamma \in [0, 1]\) is the
discount rate and \(T\) is the terminal time~\cite{sutton2018reinforcement}; we
require that either \(\gamma < 1\) or \(T < \infty\) to keep values finite.
Value functions are then defined as expectations over the returns generated
from a given state \(v_\pi(s) \defeq \E_\pi[ G_t \mid S_t = s]\), or
state-action pair \(q_\pi(s, a) \defeq \E_\pi[ G_t \mid S_t = s, A_t = a]\).
Here, the expectations with subscript \(\pi\) are taken with respect to the
implied trajectory distribution; where \(\T\) is usually omitted from
\(\pi_{\T}\) for clarity. The goal of control in RL is to find a policy
\(\pi_{\T^\star}\) that maximises the expected return from all start states,
denoted by the reward-to-go objective \(J_R(\T) \defeq \E_\pi[G_0 \mid
d_0(\cdot)]\).

\subsection{Actor-critic methods and natural policy gradients}\label{sec:prelims:nac}
Actor-critic (AC) methods are an important class of algorithms for optimising
continuously differentiable policies. They leverage the policy gradient
theorem~\cite{sutton2000policy} to update \(\T\) in the steepest ascent
direction of \(J_R(\T)\), which is typically expressed by the derivative
\begin{equation}
    \frac{\partial J_R(\T)}{\partial \T} = \int_\mathcal{S} d_{\pi_{\T}}(s)
    \int_{\mathcal{A}(s)} \frac{\partial \pi_{\T}(a \mid s)}{\partial \T}
    q_{\pi_{\T}}(s, a)\,\mathrm{d}a\,\mathrm{d}s,
\end{equation}
where \(d_\pi(s)\) denotes the discounted state
distribution~\cite{silver2014deterministic}. Using the log-likelihood
trick~\cite{williams1992simple} we can derive sample-based estimators for this
expectation (as in REINFORCE~\cite{williams1992simple}) but these are known to
suffer from high variance.  AC methods instead replace \(q_\pi(s, a)\) with an
estimate of the action-value function, \(\hat{q}_\pi(s, a)\), parameterised by
weights \(\boldsymbol{w}_q \in \mathbb{R}^n\). This \emph{critic} is learnt
through policy evaluation and results in improved stability and sample
efficiency. This can always be done without introducing bias via compatible
function approximation (CFA), such that
\(\nabla_{\boldsymbol{w}_q}\hat{q}_\pi(s, a) = \nabla_{\T} \log{\pi_{\T}(a \mid
s)}\) and the weights \(\boldsymbol{w}_q\) minimise the mean-squared error
(MSE) between \(\hat{q}_\pi(s, a)\) and \(q_\pi(s, a)\). ``Vanilla'' policy
gradient methods like these, however, often get stuck in local
optima~\cite{kakade2001natural}. Natural gradients, denoted by
\(\widetilde{\nabla}_{\T} J(\T)\), avoid this by following the steepest ascent
direction with respect to the Fisher metric rather than in standard Euclidean
space. When combined with CFA, this gradient is given by the weights
\(\boldsymbol{w}_q\) of the critic, i.e.\ an update of the form \(\T \leftarrow
\T + \eta\,\boldsymbol{w}_q\); a rather beautiful
result~\cite{peters2008natural}.

\subsection{Constrained MDPs and \texttt{RCPO}}
Constrained MDPs are a generalisation of MDPs to problems in which the optimal
policy must also satisfy a set of behavioural
requirements~\cite{altman1999constrained}. These constraints are represented by
a penalty function $c(s, a)$ (akin to the reward function), constraint
functions \(C(s) = \E_\pi[\sum_{k=0}^T \gamma^{k} c(s, a) \mid S_{t+k} = s]\)
and \(C(s, a) = \E_\pi[\sum_{k=0}^T \gamma^{k} c(s, a) \mid S_{t+k} = s,
A_{t+k} = a]\) over the realised penalties (with some abuse of notation), and
threshold $\nu \in \mathbb{R}$.  As in~\citet{tessler2019reward}, we denote the
objective associated with the constraint function by $J_C(\T) = \E_\pi[C_\pi(s)
\mid d_0(\cdot)]$ such that the optimisation problem may be expressed by the
mean-risk model
\begin{equation}\label{eq:cmdp}
    \begin{aligned}
        \max_{\pi\in\Pi} \quad & J_R(\T), \\
        \textrm{subject to} \quad & J_C(\T) \leq \nu,
    \end{aligned}
\end{equation}
where \(\Pi\) is the space of policies. Constrained optimisation problems like
this are then typically recast as saddle-point problems by Lagrange
relaxation~\cite{bertsekas1997nonlinear}:
\begin{equation}
    \min_{\lambda \geq 0}\max_{\T} \mathcal{L}(\lambda, \T) =
    \min_{\lambda \geq 0}\max_{\T} \left[
        J_R(\T) - \lambda\cdot \left(J_C(\T) - \nu\right)
    \right],
\end{equation}
where $\mathcal{L}$ denotes the Lagrangian, and $\lambda \geq 0$ the Lagrange
multiplier. \emph{Feasible solutions} are those that satisfy the constraint,
the existence of which depends on the particular problem and choice of $c(s,
a)$ and $\nu$. Any policy that is not a feasible solution is considered
sub-optimal. Approaches to solving problems of this kind then revolve around
the derivation and estimation of the gradients of $\mathcal{L}$ w.r.t.\ the
policy parameters $\T$ and multiplier
$\lambda$~\cite{borkar2005actor,bhatnagar2012online}. Recent work
by~\citet{tessler2019reward} extended these techniques to handle \emph{general
constraints} without prior knowledge, while remaining invariant to reward
scaling. Their algorithm, named \texttt{RCPO}, is a multi-timescale stochastic
approximation algorithm with provable convergence guarantees under standard
assumptions. \texttt{RCPO}'s updates take the form
\begin{align}
    \nabla_{\T}\mathcal{L} &=
    \E_\pi\left[
        \nabla_{\T}\log\pi_{\T}(s, a)\left[q_\pi(s, a) - \lambda C_\pi(s,
        a)\right]
    \right], \label{eq:rcpo:theta}\\
    \nabla_\lambda\mathcal{L} &= \nu - C_\pi(s).
\end{align}

\section{Downside Risk Measures}
In general, the choice of $c(s, a)$ depends on the problem and desired
behaviour, which need not always be motivated by risk. For example, in robotics
problems, this may take the form of a cost applied to policies with a large
jerk or snap in order to encourage smooth motion. In economics and health
problems, the constraint is typically based on some measure of risk/dispersion
associated with the uncertainty in the outcome, such as the variance. However,
in many real world applications, it is more appropriate to consider
\emph{downside risk}, such as the dispersion of returns below a target
threshold, or the likelihood of Black Swan events.  Intuitively, we may think
of a \emph{general risk measure} as a measure of ``distance'' between risky
situations and those that are risk-free, when both favourable and unfavourable
discrepancies are accounted for equally. A \emph{downside risk measure}, on the
other hand, only accounts for deviations that contribute unfavourably to
risk~\cite{dhaene2004solvency,danielsson2006consistent}.

\subsection{Partial Moments}
Partial moments were first introduced as a means of measuring the
probability-weighted deviations below (or above) a target threshold \(\tau\).
These feature prominently in finance and statistical modelling as a means of
defining (asymmetrically) risk-adjusted metrics of
performance~\cite{sortino1994performance,farinelli2008sharpe,shapiro2014lectures,sunoj2019some}.
Our definition of partial moments, stated below, follows the original
formulation of~\citet{fishburn1977mean}.
\begin{definition}
    Let $\tau \in \mathbb{R}$ denote a target value, then the
    $m$\textsuperscript{th}-order partial moments of the random variable $X$
    about $\tau$ are given by
    \begin{equation}\label{eq:partial_moments}
        \M_\pm^m[X\mid\tau] \defeq \mp\E[(\tau - X)^m_\pm],
    \end{equation}
    where $(x)_+ \equiv \max{\{x, 0\}}$, \((x)_- \equiv \min{\{x, 0\}}\), and
    $m \in [1, \infty)$.
\end{definition}
The two quantities \(\M^m_-[X\mid\tau]\) and \(\M^m_+[X\mid\tau]\) are known as
the \emph{lower} and \emph{upper} partial moments (LPM/UPM), respectively. When
the target is chosen to be the mean --- i.e. \(\tau = \E[X]\) --- we refer to
them as the \emph{centralised partial moments}, and typically drop \(\tau\)
from the notation for brevity. For example, the semi-variance is given by the
centralised, second-order LPM: $\M_-^2\left[X\right]$. Unlike the expectation
operator, \eqref{eq:partial_moments} are non-linear functions of the input and
satisfy very few of the properties that make expected values well behaved. Of
particular relevance to this work is the fact that they are non-additive. This
presents a challenge in the context of approximation since we cannot directly
apply the Robbins-Monro algorithm~\cite{robbins1951stochastic}. As we will show
in Section~\ref{sec:prediction}, however, we can estimate an upper bound for
the first partial moment, for which we introduce the following key property:
\begin{lemma}[Subadditivity]\label{lem:subadditivity}
    Consider a pair of random variables \(X\) and \(Y\), and a fixed, additive
    target \(\tau = \tau_X + \tau_Y\). Then for \(m = 1\), the partial moment
    is subadditive in \(X\) and \(Y\):
    \begin{equation}\label{eq:subadditivity}
        \M_\pm[X + Y \mid \tau] \leq \M_\pm[X \mid \tau_X] + \M_\pm[Y \mid
        \tau_Y].
    \end{equation}
\end{lemma}
\begin{proof}
    Consider the lower partial moment, expressing the inner term as a function
    of real and absolute values $\left[|\tau - X - Y| + \tau - X - Y\right] /
    2$. By the subadditivity of the absolute function (triangle inequality), it
    follows that:
    \begin{equation}
        \left(\tau - X - Y\right)_+ = \left(\tau_X - X + \tau_Y - Y\right)_+
        \leq \left(\tau_X - X\right)_+ + \left(\tau_Y -
        Y\right)_+.
    \end{equation}
    By the linearity of the expectation operator, we arrive
    at~\eqref{eq:subadditivity}. This result may also be derived for the upper
    partial moment by the same logic.
\end{proof}

\paragraph{Motivating example.}
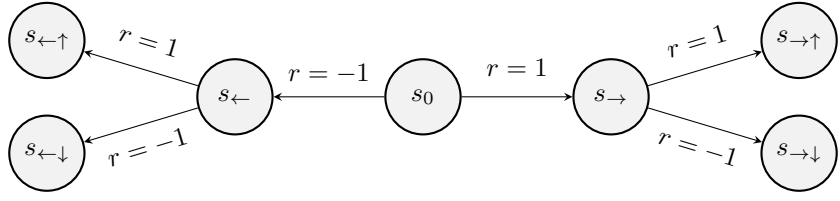
\begin{figure*}
    \centering
    \begin{tikzpicture}[
        ->, >=stealth,
        every node/.style={circle, draw, thick, fill=gray!10, minimum size=1.0cm},
        level 1/.style={sibling distance=2.8cm},
        level 2/.style={sibling distance=1.5cm},
        level distance=2.5cm
    ]
    \node (s0) {$s_0$}
        child [grow=right] {
        node (sr) {$s_\rightarrow$}
            child { node (srd) {$s_{\rightarrow\downarrow}$} edge from parent node[midway, sloped, below=-0.4cm, draw=none, fill=none, label distance=1cm] {$r=-1$} }
            child { node (sru) {$s_{\rightarrow\uparrow}$} edge from parent node[midway, sloped, above=-0.2cm, draw=none, fill=none] {$r=1$} }
            edge from parent node[midway, draw=none, fill=none, above=-0.2cm] {$r=1$}
        }
        child [grow=left] {
        node (sl) {$s_\leftarrow$}
            child { node (slu) {$s_{\leftarrow\uparrow}$} edge from parent node[midway, sloped, above=-0.2cm, draw=none, fill=none] {$r=1$} }
            child { node (sld) {$s_{\leftarrow\downarrow}$} edge from parent node[midway, sloped, below=-0.4cm, draw=none, fill=none] {$r=-1$} }
            edge from parent node[midway, draw=none, fill=none, above=-0.4cm] {$r=-1$}
        };
    \end{tikzpicture}

    \caption{Simple MDP with two actions and 7 states; the terminal state is
    omitted.}\label{fig:tamar_mdp}
\end{figure*}

\begin{figure}
    \centering
    \begin{subfigure}[t]{0.24\linewidth}
        \centering
        \includegraphics[width=\linewidth]{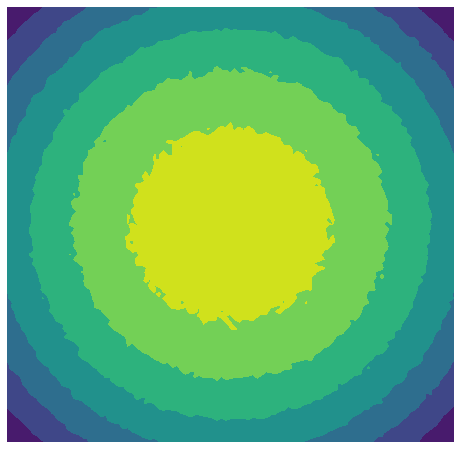}
        \caption{$\mathbb{V}[G]$}
    \end{subfigure}%
    \hfill
    \begin{subfigure}[t]{0.24\linewidth}
        \centering
        \includegraphics[width=\linewidth]{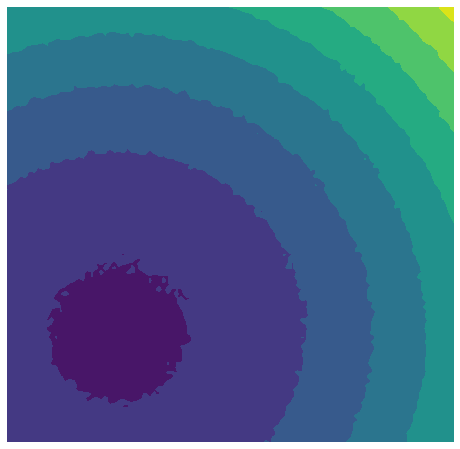}
        \caption{$\E[G] - \mathbb{V}[G]$}
    \end{subfigure}%
    \hfill
    \begin{subfigure}[t]{0.24\linewidth}
        \centering
        \includegraphics[width=\linewidth]{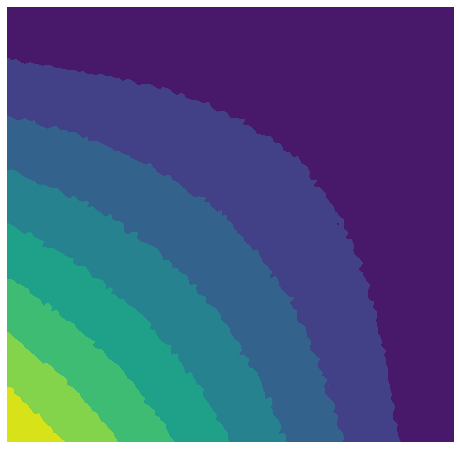}
        \caption{$\M_-[G \mid 0]$}
    \end{subfigure}
    \hfill
    \begin{subfigure}[t]{0.24\linewidth}
        \centering
        \includegraphics[width=\linewidth]{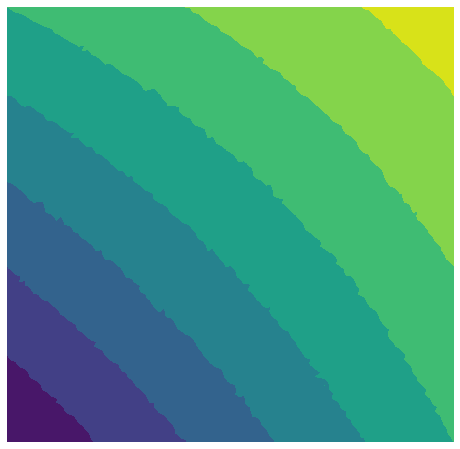}
        \caption{$\E[G] - \M_-[G \mid 0]$}
    \end{subfigure}

    \caption{Moments of the return $G$ generated by the MDP in
    Fig.~\ref{fig:tamar_mdp}. The x-axis corresponds to $\theta_1 \in [0, 1]$,
and the y-axis to $\theta_2 \in [0, 1]$. Higher values are in yellow, and lower
values in dark blue.}\label{fig:mdp_example}
\end{figure}

Why is this so important? Consider the MDP in Fig.~\ref{fig:mdp_example}, with
stochastic policy parametrised by $\T_{1,2} \in [0, 1]$ such that
$\pi(\rightarrow \mid s_0) = \theta_1$ and $\pi(\uparrow \mid s_\leftarrow) =
\pi(\uparrow \mid s_\rightarrow) = \theta_2$. As shown
by~\citet{tamar2012policy}, even in a simple problem such as this, the space of
solutions for a mean-variance criterion is non-convex. Indeed,
Fig.~\ref{fig:mdp_example} shows that the solution space exhibits local-optima
for the deterministic policies $\T_{1,2} \in \{(0, 0), (1, 0), (0,
1)\}$.\footnote{These correspond to the three minima in variance seen in Fig.~1
of~\citet{tamar2012policy}.} On the other hand, the lower partial moment only
exhibits a single optimum at the correct solution of $\T_1 = \T_2 = 1$. While
this is certainly not proof that such phenomena occur in all cases, it does
suggest that partial moments have a valid place in risk-averse RL, and may in
some instances lead to more amenable learning.

\section{Prediction}\label{sec:prediction}
Our objective in this section is now to derive an incremental,
temporal-difference prediction algorithm for the first LPM of the return
distribution \(G_t\).\footnote{The same follows for the UPM, though it's
validity in promoting risk-sensitivity is unclear.} To begin, let
$\rho[\tau](s, a)$ denote the first LPM of $G_t$ with respect to a target
function \(\tau(s, a)\), starting from state-action pair \((s, a)\), by
\begin{equation}\label{eq:rho:return}
    \rho[\tau](s, a) \defeq \M_-[G_t \mid S_t = s, A_t = a, \tau(s, a)],
\end{equation}
where the \emph{centralised moments} are shortened to \(\rho(s, a)\). For a
given target, this function can be learnt trivially through Monte-Carlo (MC)
estimation using batches of sample trajectories. Indeed, we can even learn the
higher-order moments using such an approach. However, while this yields an
unbiased estimate of the LPM, it comes at the cost of increased variance and
decreased sample efficiency~\cite{sutton2018reinforcement}. This is especially
pertinent in risk-sensitive RL which is often concerned with (already) highly
stochastic domains. The challenge is that~\eqref{eq:rho:return} is a non-linear
function of \(G_t\) which does not have a direct recursive form amenable to
TD-learning.

Rather than learn the LPM directly, we instead learn a proxy in the form of an
upper bound. To begin, we note that by Lemma~\ref{lem:subadditivity}, the LPM
of the return distribution satisfies
\begin{equation}\label{eq:rho:bound}
    \rho[\tau](s, a) \leq \M_-\left[R_{t+1} \mid \tau_R(s, a)\right] +
    \gamma\M_-\left[G_{t+1} \mid \tau(s', a')\right],
\end{equation}
for \(\tau(s, a) = \tau_R(s, a) + \gamma \E_\pi[\tau(s', a')]\). Unravelling
the final term ad infinitum yields a geometric series of bounds on the reward
moments. This sum admits a recursive form
\begin{equation}\label{eq:rho:recursive}
    \varrho[\tau](s, a) \defeq \M_-[R_{t+1} \mid \tau_R] +
    \gamma\varrho[\tau](s', a'),
\end{equation}
which is, precisely, an action-value function with non-linear reward
transformation: \(g(r) = {(\tau_R - r)}_+\).\footnote{We note that this
    expression bears a resemblance to the reward-volatility objective
    of~\citet{bisi2019risk}.} This means \emph{we are free to use any
prediction algorithm to perform the actual TD updates}, such as SARSA or
GQ(\(\lambda\))~\cite{maei2010gq}. We now only need to choose \(\tau_R\) to
satisfy our requirements; perhaps to minimise the error
between~\eqref{eq:rho:return} and~\eqref{eq:rho:recursive}. For example, a
fixed target yields the expression \(\tau_R(s, a) = (1 - \gamma) \tau\).
Alternatively, a centralised variant would be given by \(\tau_R(s, a) = r(s,
a)\). This freedom to choose a target function affords us a great deal of
flexibility in designing downside risk metrics.

\paragraph{Convergence.}
As noted by~\citet{van2019general}, Bellman equations with non-linear reward
transformations (i.e.~\eqref{eq:rho:recursive}) carry over all standard
convergence results under the assumption that the transformation is bounded.
This is trivially satisfied when the rewards themselves are
bounded~\cite{bertsekas1996neuro}. This means that the associated Bellman
operator is a contraction, and that the proxy converges with stochastic
approximation under the standard Robbins-Monro
conditions~\cite{robbins1951stochastic}.

\begin{lemma}\label{lem:var_bound}
    The variance on a random variable \(X \in \mathbb{R}\) satisfies the
    inequality \(\mathbb{V}[(c - X)_+] \leq \mathbb{V}[X]\) for arbitrary
    constant \(c \in \mathbb{R}\).
\end{lemma}

\paragraph{Variance analysis.}
Lemma~\ref{lem:var_bound} (proof of which is in the appendix) can
be used to show that the non-linear reward term in~\eqref{eq:rho:recursive}
exhibits a lower variance than that of the original reward: \(\mathbb{V}[g(R)]
\leq \mathbb{V}[R]\). While we provide no formal proof, this observation
suggests that the variances on the corresponding Bellman equations satisfy an
inequality in the same direction.  This motivates the use of our proposed LPM
proxy compared to prediction methods of higher order moments, which, by
definition, will suffer from increased variance, and may therefore be less
stable.

\section{Control}
In the previous section we saw how the upper bound on the LPM of the return can
be learnt effectively in an incremental fashion. Putting this to use now
requires that we integrate our estimator into a constrained policy optimisation
framework. This is particularly simple in the case of \texttt{RCPO}, for which
we incorporate~\eqref{eq:rho:recursive} into the penalised reward function
introduced in Def.~3~\cite{tessler2019reward}. Following their template, we may
derive a whole class of actor-critic algorithms that optimise for a downside
risk-adjusted objective, including those that make use of natural gradients.

\begin{theorem}[Compatible function approximation]\label{thm:cfa}
    If both \(\hat{q}(s, a)\) and \(\hat{\varrho}[\tau](s, a)\) are compatible
    \begin{equation}\label{eq:cfa:req1}
        \psi(s, a) \defeq \frac{\partial\hat{q}(s,
            a)}{\partial\boldsymbol{w}_q} =
            \frac{\partial\hat{\varrho}[\tau](s,
                a)}{\partial\boldsymbol{w}_\varrho} =
                \frac{1}{\pi_{\T}(a \mid s)}
                \frac{\partial\pi_{\T}(a \mid
                s)}{\partial\T},
    \end{equation}
    and independently minimise the errors
    \begin{equation}\label{eq:cfa:req2}
        \mathcal{E}_q^2 \defeq \E\left[{\left(\psi{(s, a)}^\top \cdot
        \boldsymbol{w}_q - q(s, a)\right)}^2\right], \quad
        \mathcal{E}_\varrho^2 \defeq \E\left[{\left(\psi{(s, a)}^\top \cdot
        \boldsymbol{w}_\varrho - \varrho[\tau](s, a)\right)}^2\right],
    \end{equation}
    then we can replace \(q(s, a) - \lambda\varrho[\tau](s, a)\) with
    \(\hat{q}(s, a) - \lambda\hat{\varrho}[\tau](s, a)\) to get
    \begin{equation}\label{eq:rcpo:pgrad}
        \widetilde{\nabla}_{\T}\mathcal{L}(\lambda, \T) = \frac{\partial
            \mathcal{L}(\lambda, \T)}{\partial\T} = \int_\mathcal{S}
            d_{\pi_{\T}}(s) \int_{\mathcal{A}(s)} \frac{\partial \pi_{\T}(a
                \mid a)}{\partial\T} \left[ \hat{q}(s, a) - \lambda
            \hat{\varrho}[\tau](s, a) \right]\,\mathrm{d}a\,\mathrm{d}s.
    \end{equation}
\end{theorem}

Crucially, if the two value function estimators \(\hat{q}(s, a)\) and
\(\hat{\varrho}[\tau](s, a)\) are \emph{compatible} with the policy
parameterisation~\cite{sutton2000policy}, then we may extend \texttt{RCPO} to
use natural policy gradients (see Section~\ref{sec:prelims:nac}). We call the
resulting algorithm \texttt{NRCPO} for which the existence hinges on
Theorem~\ref{thm:cfa} above; a proof is provided in the appendix.
This shows under which conditions the true value functions may be replaced with
the function approximators, without introducing bias. Assuming the conditions
are met, it remains only to replace the ``vanilla'' gradient in
Eq.~\ref{eq:rcpo:theta} with the natural gradient. This yields a policy update
\[
    \T \leftarrow \T + \eta \left(\boldsymbol{w}_q -
    \lambda\boldsymbol{w}_\varrho\right),
\]
which is not only trivial to implement, but also benefits from all the
advantages associated with using natural gradients~\cite{kakade2001natural}.

\section{Numerical Experiments}
Here we present evaluations of our proposed \texttt{NRCPO} algorithm on three
experimental domains using variations on \(\tau(s, a)\) in the LPM proxy
(Eq.~\ref{eq:rho:recursive}). The chosen hyperparameters and experimental
configurations, unless otherwise stated, can be found in the appendix.

\subsection{Multi-armed bandit}
The first problem setting --- taken from~\citet{tamar2015policy} --- is a
\emph{3-armed bandit} with rewards distributed according to: $R_A \sim
\mathcal{N}(1, 1)$; $R_B \sim \mathcal{N}(4, 6)$; and $R_C \sim
\textrm{Pareto}(1, 1.5)$. The expected reward from each arm is 1, 4 and 3,
respectively. The optimal solution for a risk-neutral agent is to choose the
second arm, but it is apparent that agents sensitive to negative values should
choose the third arm since the Pareto distribution's support is bounded from
below.

\paragraph{Results.}
We evaluated our proposed methods by training three different Boltzmann
policies on the multi-armed bandit problem. The first
(Fig.~\ref{fig:bandit:mean}) was trained using a standard variant of NAC, the
latter two (Figs.~\ref{fig:bandit:mean_lpm} and~\ref{fig:bandit:mean_lpm2})
used a stateless version of \texttt{NRCPO} with first and second LPMs as risk
measures, respectively; for simplicity, we assume a constant value for the
Lagrange multiplier \(\lambda\).  The results show that after $\sim 5000$
samples, both risk-averse policies have converged on arm $C$. This highlights
the flexibility of our approach, and the improvements in efficiency that can be
gained from incremental algorithms compared to Monte-Carlo estimation. See, for
example, the approach of \citet{tamar2015policy} which used \emph{10000 samples
per gradient estimate}, requiring $\sim 10^5$ sample trajectories before
convergence.

\begin{figure*}
    \centering
    \begin{subfigure}[t]{0.38\linewidth}
        \centering
        \includegraphics[width=\linewidth]{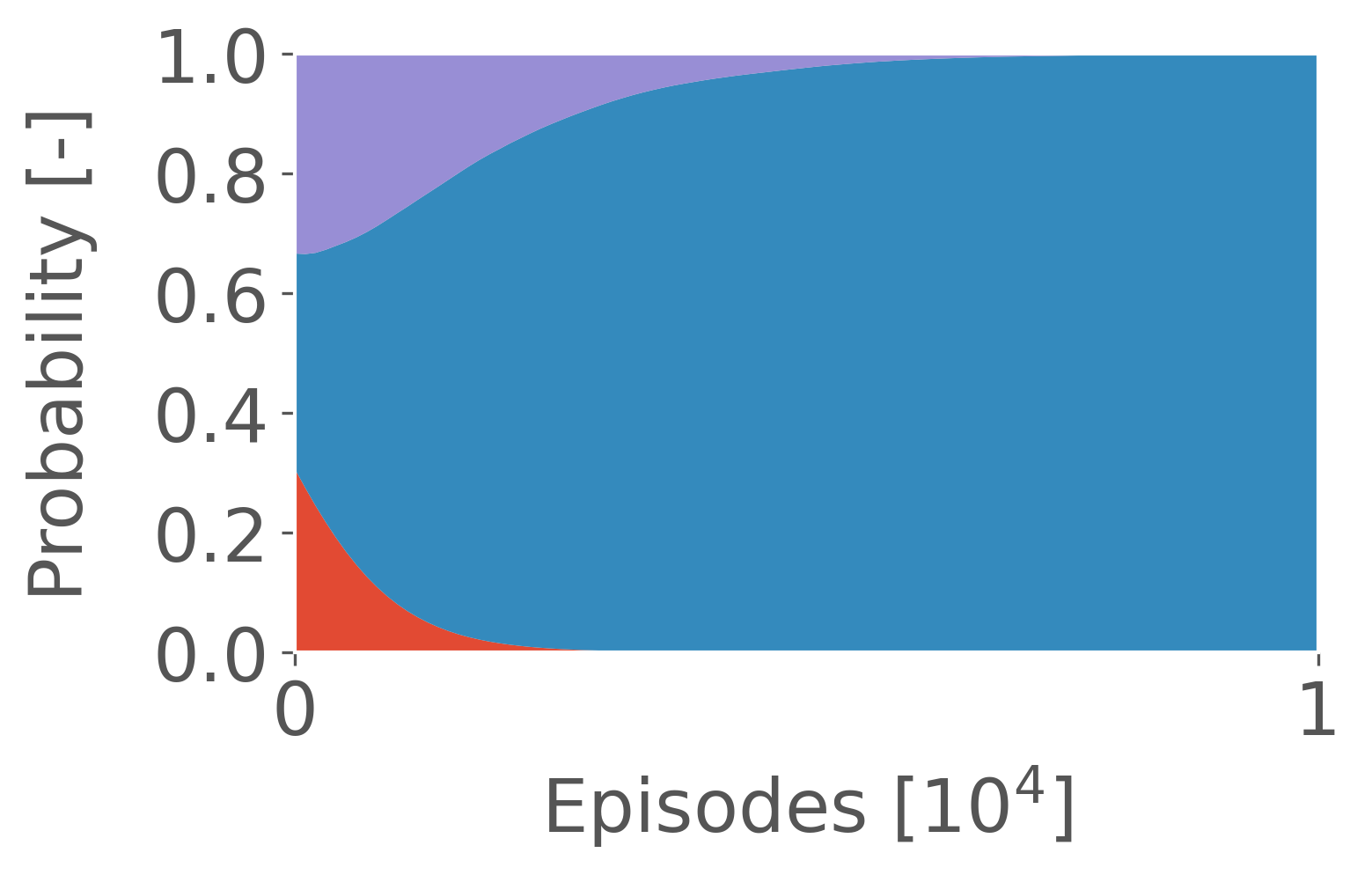}
        \caption{\(\E[R]\)}\label{fig:bandit:mean}
    \end{subfigure}%
    \hfill
    \begin{subfigure}[t]{0.31\linewidth}
        \centering
        \includegraphics[width=\linewidth]{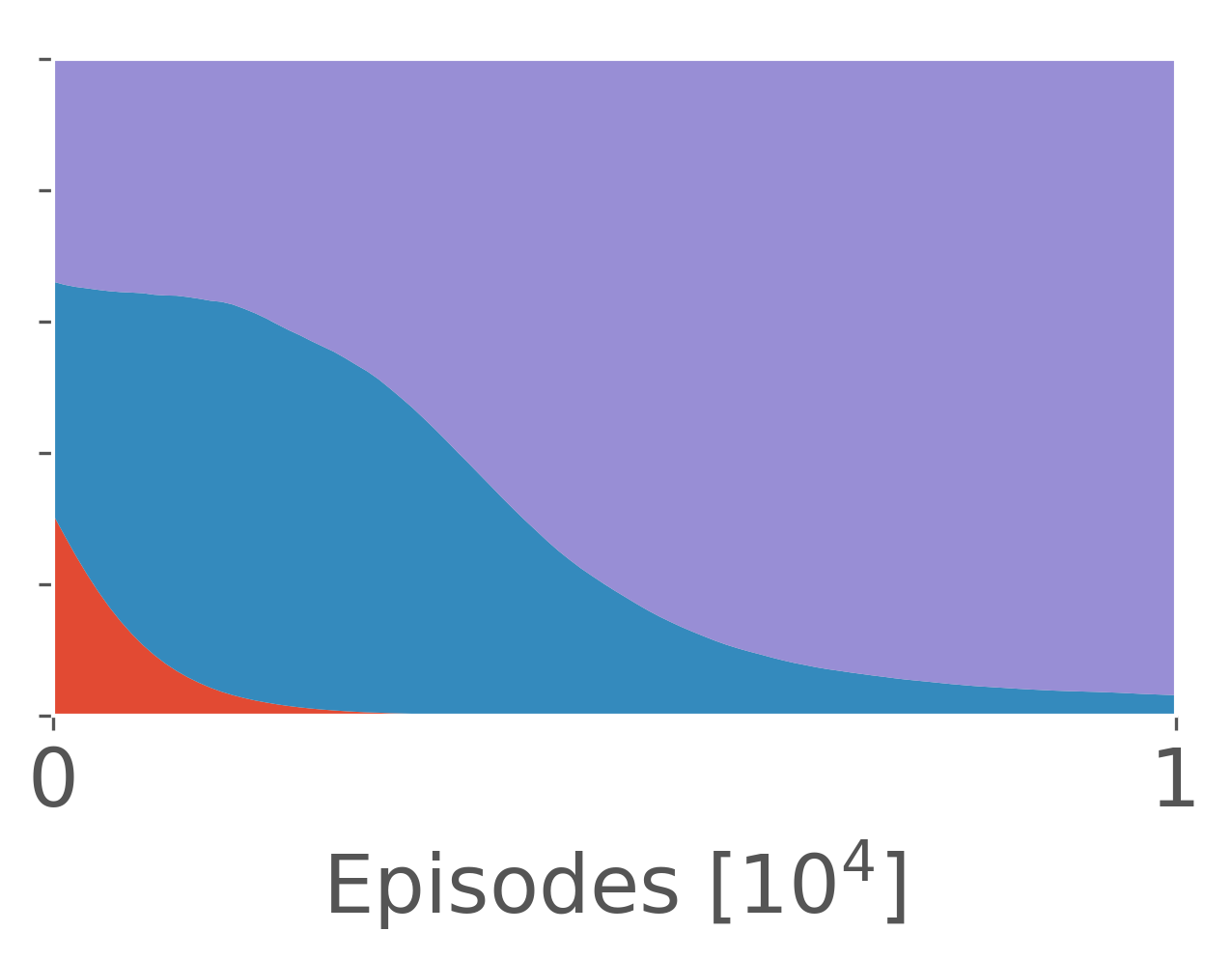}
        \caption{\(\E[R] - 2\M_-[R]\)}\label{fig:bandit:mean_lpm}
    \end{subfigure}%
    \hfill
    \begin{subfigure}[t]{0.31\linewidth}
        \centering
        \includegraphics[width=\linewidth]{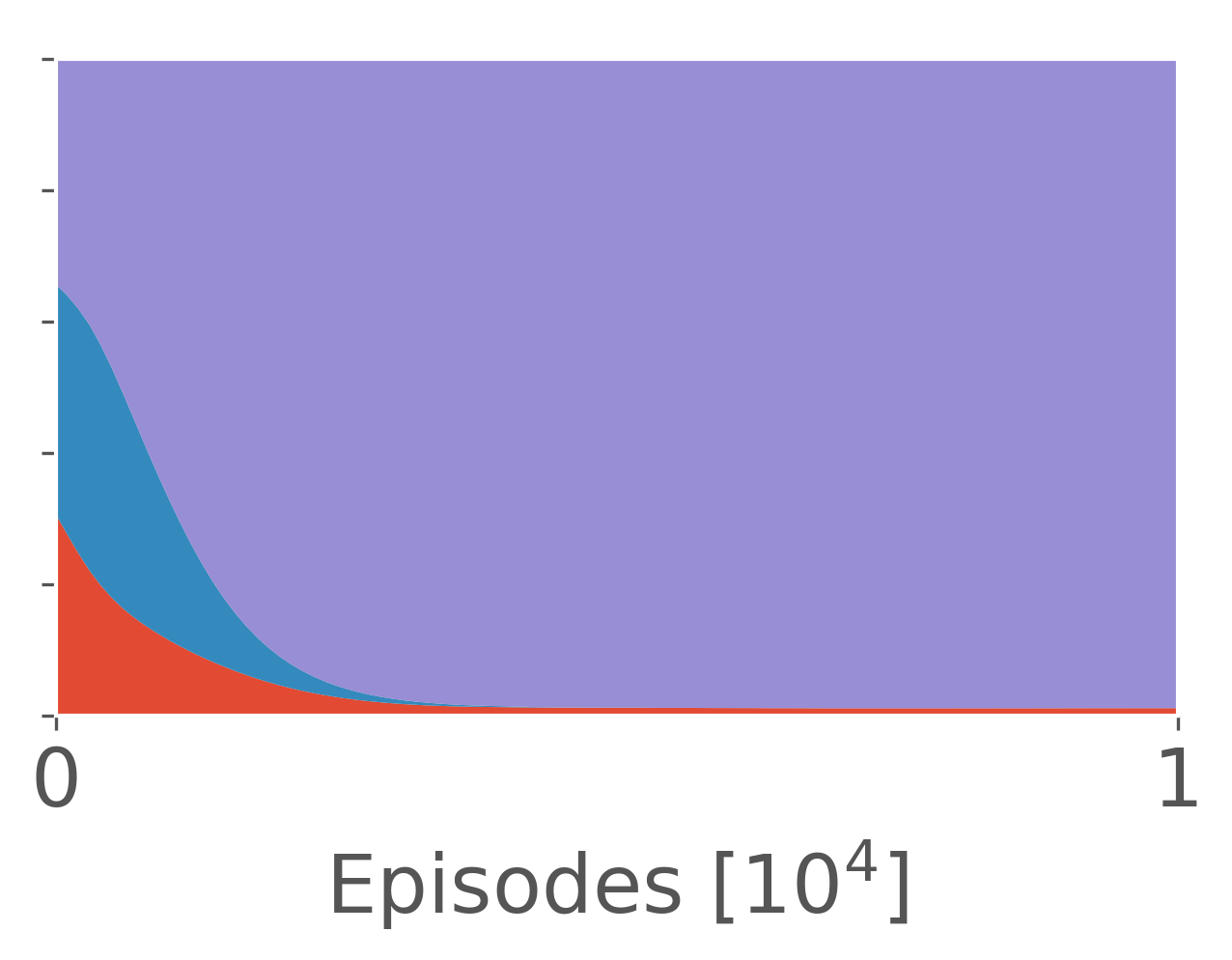}
        \caption{\(\E[R] - \M_-^2[R]\)}\label{fig:bandit:mean_lpm2}
    \end{subfigure}

    \caption{Evolution of Boltzmann policies' selection probabilities for arms
    A (red), B (blue) and C (purple). Each curve represents a normalised
average over 100 independent trials.}\label{fig:bandit}
\end{figure*}

\subsection{Portfolio optimisation}\label{sec:experiments:portfolio}
The next setting is an adaptation of the portfolio optimisation problem first
proposed by~\citet{tamar2012policy}, and featured in~\cite{bisi2019risk}, as a
motivating example of risk-averse RL in finance. Consider two types of asset: a
liquid asset such as cash holdings with interest rate $r_\text{L}$; and an
illiquid asset with time-dependent interest rate $r_\text{N}(t) \in
\{\overline{r}_\text{N}, \underline{r}_\text{N}\}$ that switches between two
values stochastically. Unlike the original formulation, we do not assume that
this happens symmetrically. Instead, we treat \(r_\text{N}(t)\) as a switching
process with two states and transition probabilities $p_\uparrow$ and
$p_\downarrow$. At each timestep the agent chooses an amount (up to $M$) of the
illiquid asset to purchase at a fixed cost per unit \(\alpha\). At maturity
(after $N$ steps) this illiquid asset either defaults (with probability
$p_\text{D}$), or is sold and converted into liquid asset. The state space of
the problem is embedded in $\mathbb{R}^{N+2}$, where the first entry denotes
the allocation in the liquid asset, the next $N$ are the allocations in the
non-liquid assets, and the final value is given by $r_\text{N}(t) - \E_{t' <
t}[r_\text{N}(t')]$. The actions are discrete choices for the purchase order,
and the reward at each timestep is given by the log-return in the liquid asset
between $t$ and $t+1$, as in~\cite{bisi2019risk}.

\paragraph{Results.}
\begin{figure*}
    \centering
    \begin{subfigure}[t]{0.48\linewidth}
        \centering
        \includegraphics[width=\linewidth]{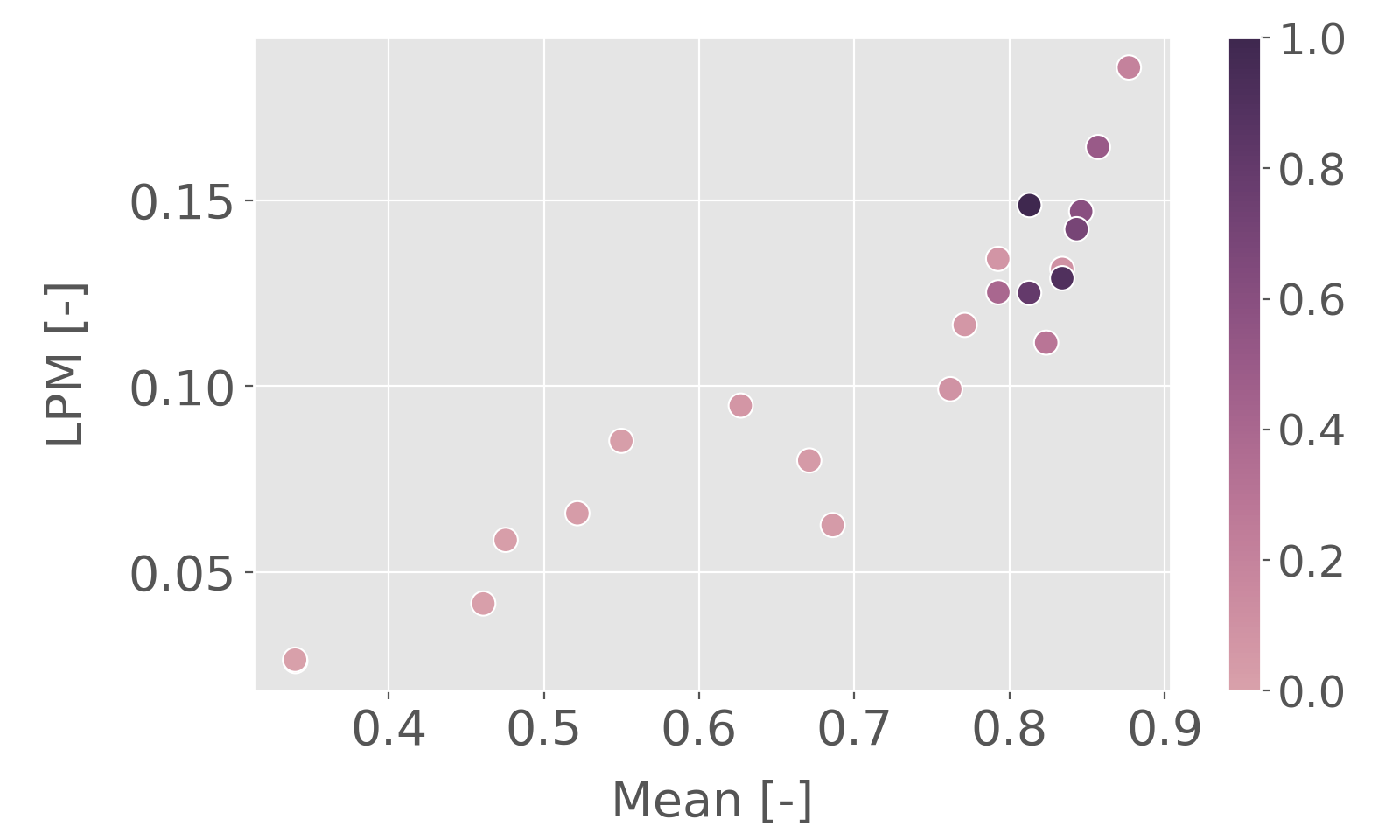}
        \caption{Mean vs.\ the LPM of returns.}
    \end{subfigure}%
    \hspace{1em}
    \begin{subfigure}[t]{0.445\linewidth}
        \centering
        \includegraphics[width=\linewidth]{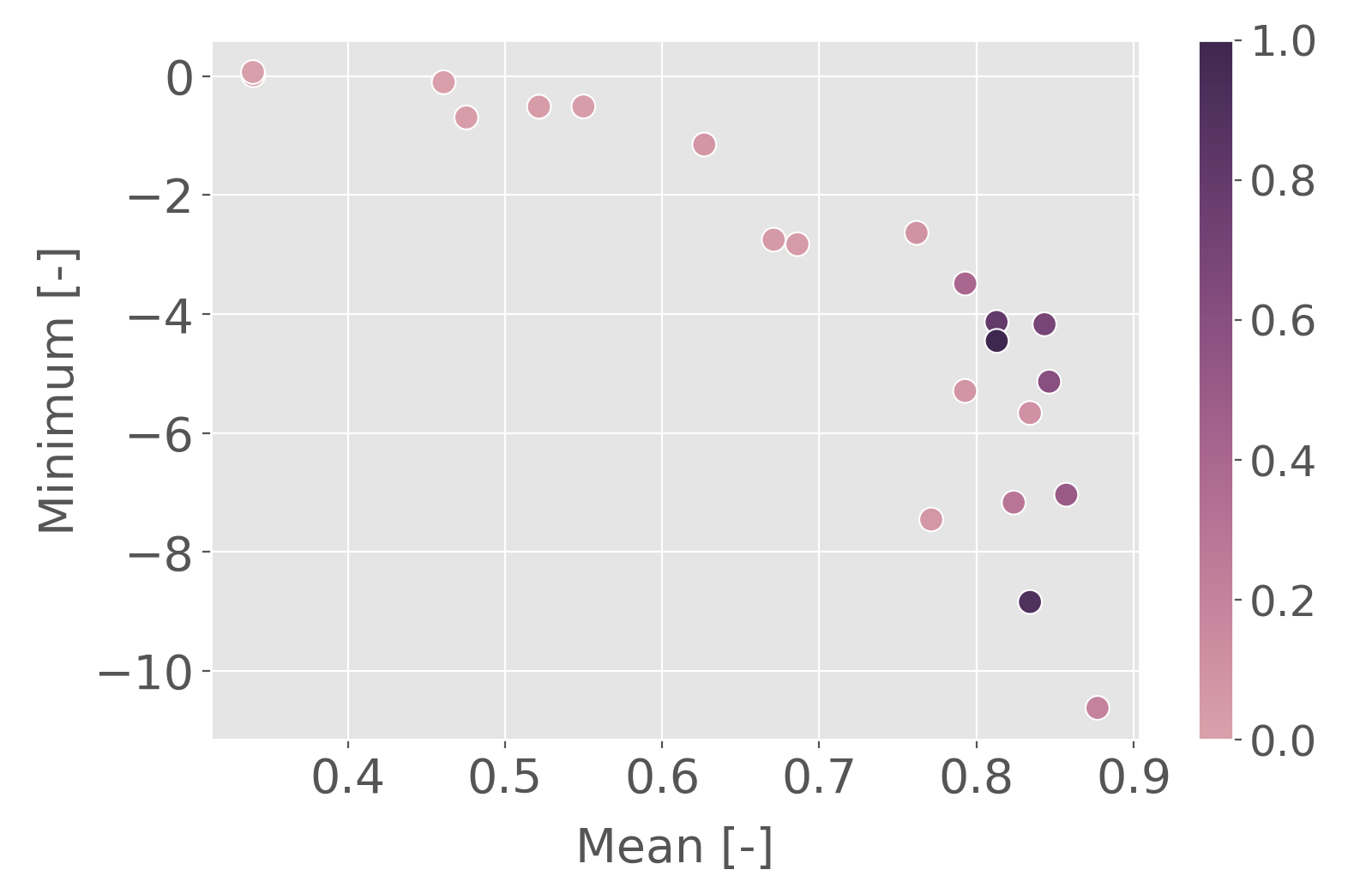}
        \caption{Mean vs.\ the minimum observed return.}
    \end{subfigure}%

    \caption{Performance of portfolio optimisation solutions for varying
    thresholds \(\nu \in [0, 1]\). Each point was generated by evaluating the
    policy over 10\textsuperscript{4} trials following
    training.}\label{fig:portfolio}
\end{figure*}

Figure~\ref{fig:portfolio} shows how the performance of our LPM variant of
\texttt{NRCPO} performs on the portfolio optimisation problem; in this case we
chose to use the centralised LPM as a target, i.e.\ set \(\tau_R(s, a) = r(s,
a)\).  We observe the emergence of a ``frontier'' of solutions which trades-off
maximisation of the expected return with minimisation of the risk penalty. As
the threshold \(\nu\) (see Eq.~\ref{eq:cmdp}) increases (i.e.\ increasing
tolerance to risk), so too do we see a tendency for solutions with a higher
mean, higher LPM and more extreme minima. From this we can conclude that
minimisation of the proxy~\eqref{eq:rho:recursive} does have the desired effect
of reducing the LPM, validating the \emph{practical value of the bound}.

\subsection{Optimal consumption}
The final setting, known as \emph{Merton's optimal consumption
    problem}~\cite{merton1969lifetime} is another example of intertemporal
    portfolio optimisation. This particular setting has been largely unstudied
    in the RL literature\footnote{To the best of our knowledge, the work of
    \citet{weissensteiner2009q} is the only prior example.} in spite of the
    fact that it represents a broad class of real-world problems; e.g.\
    retirement planning. At each timestep the agent must specify two
    quantities:
\begin{enumerate*}
    \item the proportion of it's wealth to invest in a risky asset (whose
        returns we assume to be Normally distributed) and a risk-free/liquid
        asset with deterministic return \(r_\text{L}\), and
    \item an amount of it's wealth to consume and permanently remove from
        the portfolio.
\end{enumerate*}
The problem terminates when all the agent's wealth, \(W_t\), is consumed or the
terminal timestep is reached (200 steps). In the latter case, any remaining
wealth that wasn't consumed is lost. The state space is given by the current
time and the agent's remaining wealth, and the reward is defined to be the
amount consumed. To highlight downside risk, we also extend the traditional
model to include defaults. At each decision point there is a non-zero
probability \(p_\text{D}\) that the risky asset's underlying ``disappears'',
the risky investment is lost, the problem terminates, and the remaining wealth
in the liquid asset is consumed in it's entirety.\footnote{Similar variations
    on the original problem have been considered in the literature by,
e.g.,~\citet{puopolo2017portfolio}.}

\paragraph{Results.}
In this case we defined a custom target function by leveraging prior knowledge
of the problem. Specifically, we set \(\tau_R(s, a) = W_t \cdot \Delta t (T -
t)\), where \(t\) and \(T\) denote the current and terminal times, and \(\Delta
t\) is the time increment. This has the interpretation of the expected reward
generated by an agent that consumes it's wealth at a fixed rate.  Unrolling the
recursive definition of \(\tau(s, a)\), we have an implied target of \(W_t\)
for all states. In other words, we associate a higher penalty with those
policies that underperform said reasonable ``benchmark'' and finish the episode
having consumed less wealth than the initial investment.

\begin{figure*}
    \centering
    \begin{subfigure}[t]{0.47\linewidth}
        \centering
        \includegraphics[width=\linewidth]{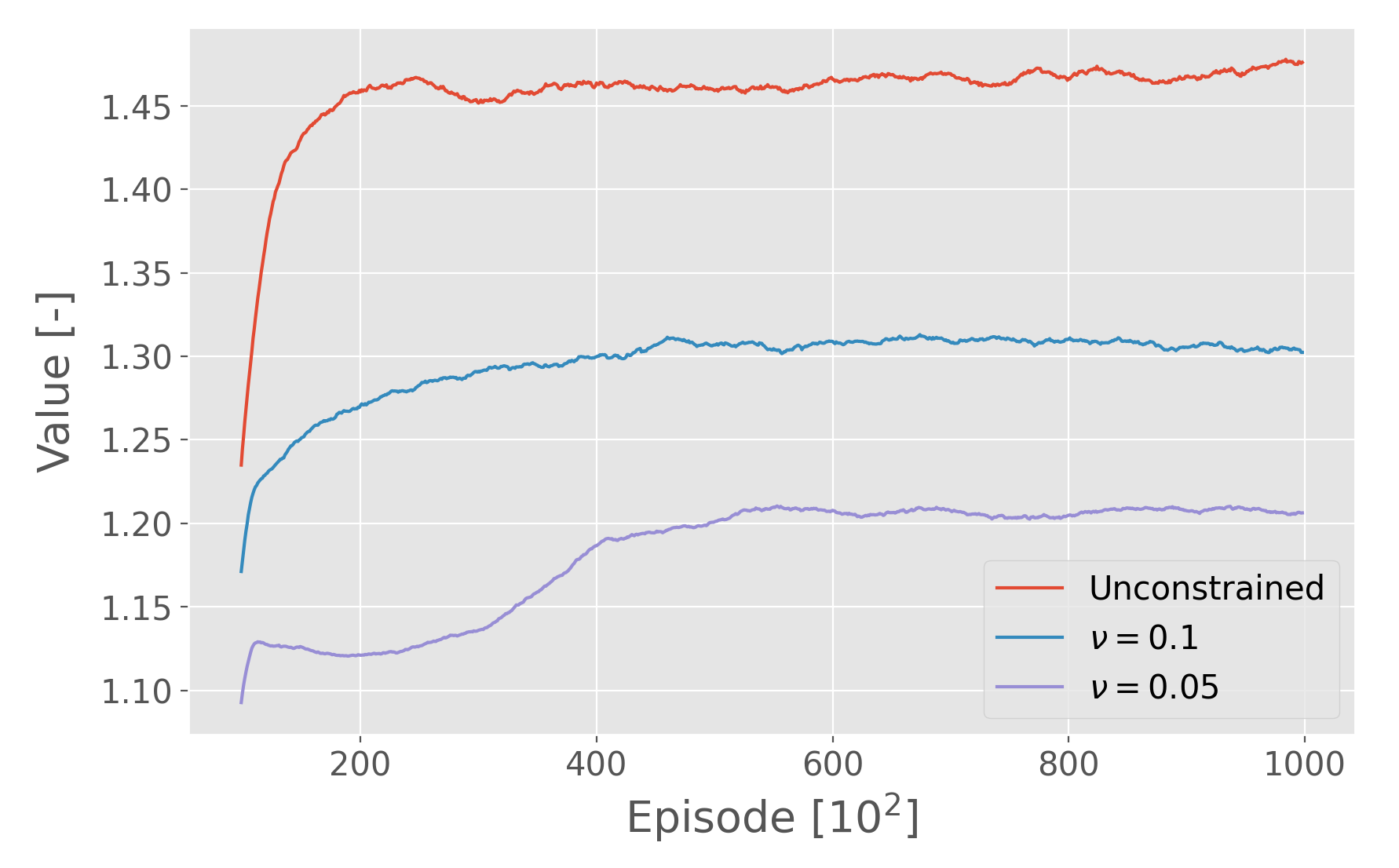}
        \caption{\(J_R(\T)\)}
    \end{subfigure}%
    \hspace{1em}
    \begin{subfigure}[t]{0.47\linewidth}
        \centering
        \includegraphics[width=\linewidth]{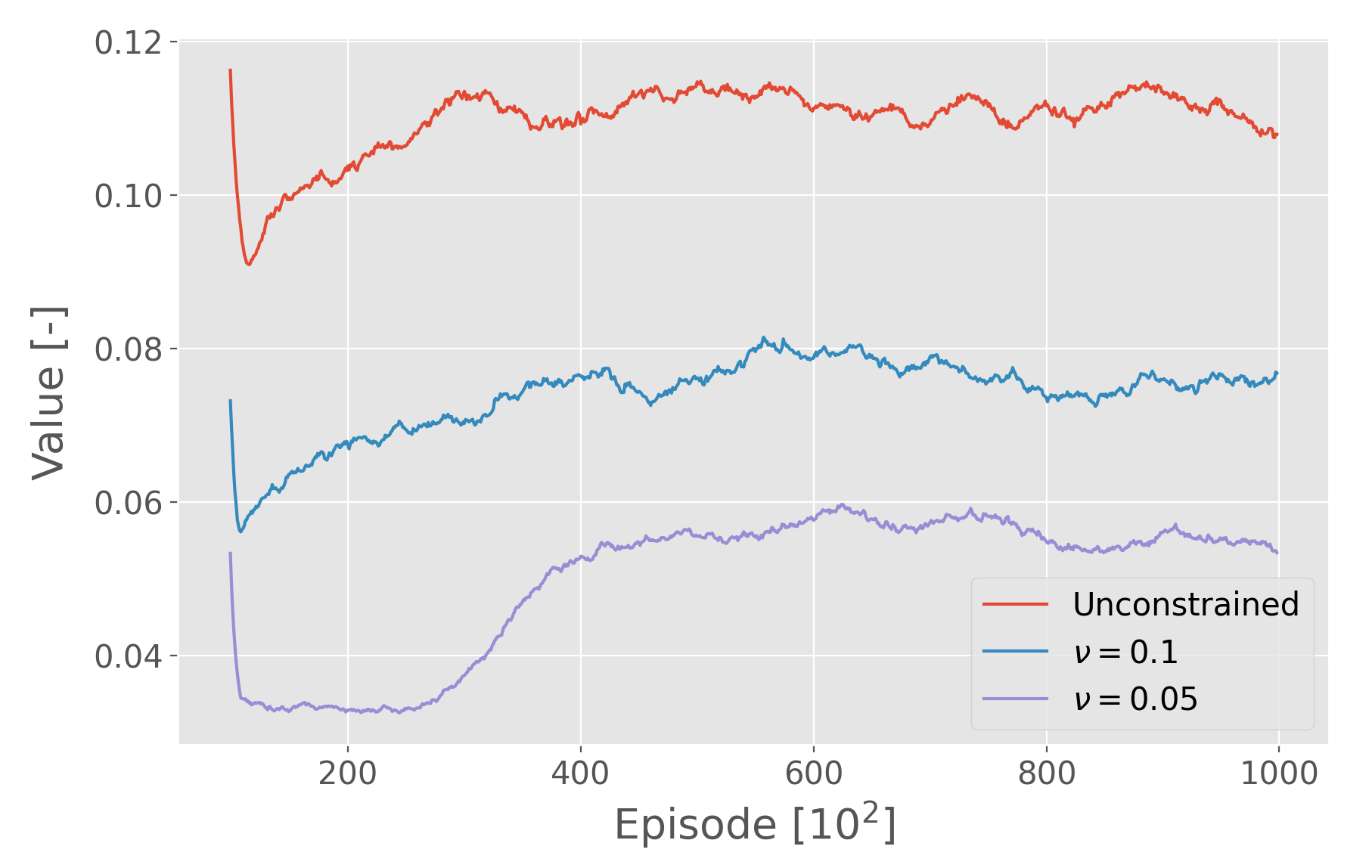}
        \caption{\(J_C(\T)\)}
    \end{subfigure}%

    \caption{Evolution of performance of optimal consumption solutions for
    \(\nu \in \{0.05, 0.1, \infty\}\). Each curve was generated by evaluating
the policy for 100 trials every 100 training episodes, with a simple moving
average of period 100 applied for clarity.}\label{fig:consumption}
\end{figure*}

Figure~\ref{fig:consumption} shows how performance of our algorithm evolved
during training. Each curve was generated by sampling 100 trajectories every
100 training episodes to estimate statistics. As in the previous section, we
observe how decreasing \(\nu\) leads to increasing risk-aversion in the form of
a lower mean and LPM. In all cases the algorithm was able to identify a
\emph{feasible solution} and exhibited highly stable learning. An important
conclusion to take from this is that the flexibility to choose \(\tau_R(s, a)\)
affords us a great deal of control over the behaviour of the policy. In this
case, we \emph{only penalise downside risk associated with losses}.
Furthermore, \texttt{(N)RCPO} removes the need for calibrating the multiplier
\(\lambda\), which can be very hard to tune~\cite{achiam2017constrained}. This
makes our approach \emph{highly practical for many real-world problems}.

\section{Conclusions}
In this paper we have put forward two key ideas. First, that partial moments
offer a tractable alternative to conventional metrics such as variance or
conditional value at risk. We show that our proxy has a simple interpretation
and enjoys favourable reward variance. Second, we demonstrate how an existing
method in constrained policy optimisation can be extended to leverage natural
gradients, an algorithm we call \texttt{NRCPO}. The combination of these two
developments is a methodology for deriving downside risk-averse policies with a
great deal of flexibility and sample efficiency. In future work we hope to
address questions on computational complexity, and how these methods may be
applied to multi-agent systems.

\paragraph{Software and Data.}
All our code will be made freely accessible via GitHub.

\small
\bibliographystyle{unsrtnat}
\bibliography{refs}

\clearpage
\appendix

\section{Variance Analysis}
\begin{customlemma}{2}
	For any constant \(c \in \mathbb{R}\), the variance, $\mathbb{V}[X]$, of a random variable $X$ that is supported on
	$\mathbb{R}$ satisfies $\mathbb{V}[(c - X)_+] \leq \mathbb{V}[X].$
	
\end{customlemma}
\begin{proof}
    Let $Y \defeq c - X$ and $Z \defeq (Y)_+ = (c - X)_+$. Since $c$ is a constant,
     $\mathbb{V}[Y]=\mathbb{V}[X]$. Thus, it suffices to show that
    \(\mathbb{V}[Z] \leq \mathbb{V}[Y]\).

	Denote by  \(p_+ \defeq \mathbb{P}(Y \geq
    0)\) and \(p_- \defeq \mathbb{P}(Y < 0)\),
    the probabilities of \(Y\) being
    non-negative or negative respectively. 
    Let $\mu_+ \defeq \E[Y \mid Y \geq 0]$ and $\mu_- \defeq \E[Y \mid Y < 0]$.
    Then we have:
    \begin{align}
        \mu_Y \defeq \E[Y] &= p_+ \mu_+ + p_- \mu_-, \\
        \mu_Z \defeq \E[Z] &= p_+ \mu_+.
    \end{align}
    Similarly,
    with $m_+^2 \defeq \E[Y^2 \mid Y \geq 0]$ and $m_-^2 \defeq \E[Y^2 \mid Y < 0]$,
    the second moments are:
    \begin{align}
        m_Y^2 \defeq \E[Y^2] &= p_+ m^2_+ + p_- m^2_-, \\
        m_Z^2 \defeq \E[Z^2] &= p_+ m^2_+,
    \end{align}

    Then we can express the variance of $Y$ and $Z$ as
    \begin{align}
		\mathbb{V}[Y] &= p_+ m^2_+  - p_+^2 \mu^2_+ + p_- m^2_- - 2p_+p_-\mu_+\mu_- - p_-^2 \mu^2_-,\\
		\mathbb{V}[Z] &= p_+ m^2_+  - p_+^2 \mu^2_+.
    \end{align}
    Their difference is
    \begin{equation}
        \mathbb{V}[Y] - \mathbb{V}[Z] = p_- m^2_- - 2p_+p_-\mu_+\mu_- - p_-^2
        \mu^2_-.
    \end{equation}
    We split this difference into two non-negative terms as follows.
    \begin{align}
        p_- m^2_- - p_-^2 \mu^2_- &\geq 0, \\
        -2p_+p_-\mu_+\mu_- &\geq 0.
    \end{align}
	The first inequality follows from the fact that $\mathbb{V}[Y \mid Y < 0] =  m^2_- - \mu^2_- \ge 0$, so $m^2_- \geq \mu^2_-$ and 
	$p_- \in [0, 1]$.
	The second inequality follows from the facts that $p_+,p_-,\mu_+\ge 0$ and $\mu_-< 0$.
    Thus, we have shown that $\mathbb{V}[Y] - \mathbb{V}[Z] \ge 0$, which completes the proof.
\end{proof}

\section{Compatible Function Approximation}

We now show that, under certain conditions, we can use \(\hat{q}(s, a) -
\lambda\hat{\varrho}[\tau](s, a)\), instead of \(q(s, a) -
\lambda\varrho[\tau](s, a)\), to get an unbiased estimate of the policy
gradient.

\newpage

\begin{customthm}{1}[Compatible function approximation]\label{thm:cfa}
    If the following three conditions hold:

	\begin{enumerate}[leftmargin=0.5cm]
	\item \(\hat{q}(s, a)\) and \(\hat{\varrho}[\tau](s, a)\) are compatible with
		the policy $\pi_{\T}$, i.e.,
    \begin{equation}\label{eq:cfa:req1}
        \frac{\partial\hat{q}(s, a)}{\partial\boldsymbol{w}_q} =
        \frac{\partial\hat{\varrho}[\tau](s,
        a)}{\partial\boldsymbol{w}_\varrho} = \frac{1}{\pi_{\T}(a \mid s)}
        \frac{\partial\pi_{\T}(a \mid s)}{\partial\T};
    \end{equation}
	\item \(\hat{q}(s, a)\) minimises the error
    \begin{equation}\label{eq:cfa:req2}
        \mathcal{E}_q^2 \defeq \E\left[{\left(\hat{q}_{\boldsymbol{w}_q}(s, a)
        - q(s, a)\right)}^2\right],
    \end{equation}
	\item
	and \(\hat{\varrho}[\tau](s, a)\) minimises the error
    \begin{equation}\label{eq:cfa:req3}
        \mathcal{E}_\varrho^2 \defeq
        \E\left[{\left(\hat{\varrho}_{\boldsymbol{w}_\varrho}[\tau](s, a) -
        \varrho[\tau](s, a)\right)}^2\right],
    \end{equation}
	\end{enumerate}
    then 
    \begin{equation}\label{eq:rcpo:pgrad}
        \frac{\partial \mathcal{L}(\lambda, \T)}{\partial\T} =
        \int_\mathcal{S} d_{\pi_{\T}}(s) \int_{\mathcal{A}(s)}
        \frac{\partial \pi_{\T}(a \mid
            a)}{\partial\T} \left[ \hat{q}(s, a) - \lambda
        \hat{\varrho}[\tau](s, a) \right]\,\mathrm{d}a\,\mathrm{d}s
    \end{equation}
    is an unbiased estimate of the policy gradient.
\end{customthm}
\begin{proof}
    Let \(f(s, a)\) be an arbitrary function of state and action and let there
    exist a corresponding approximator \(\hat{f}(s, a)\) with weights
    \(\boldsymbol{w}_f\). The MSE between the true function and approximation
    is given by
    \begin{equation}
        \mathcal{E}_f^2 = \int_\mathcal{S} d_{\pi_{\T}}(s)
        \int_{\mathcal{A}(s)} \pi_{\T}(a \mid s) \left[\hat{f}(s, a) - f(s, a)
        \right]^2 \text{d}a \text{d}s.
    \end{equation}
    If \(\hat{f}(s, a)\) fulfils requirement~\eqref{eq:cfa:req1}, then the
    derivative of the MSE is given by
    \begin{align}
        \frac{\partial\mathcal{E}_f^2}{\partial\boldsymbol{w}_f}
            &= 2\int_\mathcal{S} d_{\pi_{\T}}(s) \int_{\mathcal{A}(s)}
            \pi_{\T}(a \mid s) \frac{\partial \hat{f}(s,
            a)}{\partial\boldsymbol{w}_f} \left[\hat{f}(s, a) - f(s, a) \right]
            \text{d}a \text{d}s, \\
            &= 2\int_\mathcal{S} d_{\pi_{\T}}(s) \int_{\mathcal{A}(s)}
            \frac{\partial \pi_{\T}(a \mid s)}{\partial\T} \left[ \hat{f}(s, a)
            - f(s, a) \right] \text{d}a \text{d}s.
    \end{align}
    If we then assume that the learning method minimises the MSE defined above
    (i.e.\ requirements~\eqref{eq:cfa:req2} and~\eqref{eq:cfa:req3}), then the
    weights \(\boldsymbol{w}_f\) give the stationary point. Equating the
    expression to zero thus results in the equality
    \begin{equation}
        \int_\mathcal{S} d_{\pi_{\T}}(s) \int_{\mathcal{A}(s)} \frac{\partial
        \pi_{\T}(a \mid a)}{\partial\T} f(s, a) \text{d}a \text{d}s =
        \int_\mathcal{S} d_{\pi_{\T}}(s) \int_{\mathcal{A}(s)} \frac{\partial
        \pi_{\T}(a \mid a)}{\partial\T} \hat{f}(s, a) \text{d}a \text{d}s,
    \end{equation}
    where \(f(s, a)\) may be replaced with either \(q(s, a)\) or
    \(\varrho[\tau](s, a)\) as needed. This means that we can interchange the
    true value functions in the policy gradient with the MSE-minimising
    approximations.  This is the classic result of compatible function
    approximation originally presented by~\citet{sutton2000policy}.

    Now, for an objective of the form \(\mathcal{L}(\lambda, \T) = J_R(\T) -
    \lambda J_C(\T)\) (i.e.\ that used in \texttt{RCPO}), we have that
    \begin{equation}\label{eq:diff_lagrange}
        \frac{\partial \mathcal{L}(\lambda, \T)}{\partial\T} = \frac{\partial
        J_R(\T)}{\partial\T} - \lambda \frac{\partial J_C(\T)}{\partial\T}.
    \end{equation}
    From the policy gradient theorem~\cite{sutton2000policy}, and the
    compatible function approximation result above, we also know that each of
    the differential terms may be expressed by an integral of the form
    \begin{equation}\label{eq:pgt}
        \frac{\partial J_f(\T)}{\partial\T} = \int_\mathcal{S} d_{\pi_{\T}}(s)
        \int_{\mathcal{A}(s)} \frac{\partial \pi_{\T}(a \mid a)}{\partial\T}
        \hat{f}(s, a)\,\mathrm{d}a\,\mathrm{d}s.
    \end{equation}
    Combining~\eqref{eq:diff_lagrange} and~\eqref{eq:pgt}, it follows from the
    linearity of integration (sum rule) that the policy gradient of the
    Lagrangian is given by~\eqref{eq:rcpo:pgrad}.
\end{proof}

\section{Experiments}
A template of our proposed algorithm \texttt{NRCPO} using the LPM as a
constraint is outlined in Algorithm~\ref{alg:nrcpo}. This implementation
includes a tweak to the traditional policy update that was first introduced
by~\citet{thomas2014bias} in which the advantage weights are rescaled using
$L_2$ normalisation. This was found to improve stability during learning.

\begin{algorithm}[H]
    \caption{\texttt{NRCPO-LPM}}
	\label{alg:nrcpo}
    \begin{algorithmic}[1]
        \FOR{\(k \leftarrow 0, 1 \ldots\)}
            \STATE Initialise state \(s_0 \sim d_0(\cdot)\)
            \FOR{\(t \leftarrow 0, 1 \ldots, T-1\)}
                \STATE Sample action \(a_t \sim \pi(\cdot \mid s_t)\)
                \STATE Observe new state \(s_{t+1}\) and reward \(R_{t+1}\)
                \\[0.5em]
            \STATE \textbf{Update critics:} \(\hat{q}(s_t, a_t)\) and
                \(\hat{\varrho}[\tau](s_t, a_t)\)
            \\[0.5em]
            \IF{\(t + 1 \bmod N_{\text{Policy}} = 0\)}
                \STATE \textbf{Update policy:} \(\boldsymbol{\theta} \leftarrow
                \eta \frac{\boldsymbol{w}_q -
                \lambda\boldsymbol{w}_\varrho}{||\boldsymbol{w}_q -
                \lambda\boldsymbol{w}_\varrho||_2}\)
            \ENDIF
            \ENDFOR
        \ENDFOR
    \end{algorithmic}
\end{algorithm}

In the following subsections, we specify for each of the three experimental
domains: the policy $\pi_{\T}$; \(\hat{q}(s, a)\) and~\(\hat{\varrho}[\tau](s,
a)\); and the prediction algorithm used to estimate the weights
\(\boldsymbol{w}_q\) and \(\boldsymbol{w}_\varrho\) in the ``Update critics''
step. We also state the values of any domain-specific parameters used during
experiments.

\subsection{Bandit}
In the bandit problem we considered a Gibbs policy of the form
\begin{equation}
    \pi_{\T}(a) = \frac{e^{\theta_a}}{\sum_b e^{\theta_b}},
\end{equation}
where each action corresponded to a unique choice over the three arms \(a \in
\{A, B, C\}\). The value functions were then represented by linear function
approximators
\begin{align}
    \hat{q}(a) &= \nabla_{\T}\log\pi_{\T}(a)^\top \boldsymbol{w}_q + v_q, \\
    \hat{\varrho}(a) &= \nabla_{\T}\log\pi_{\T}(a)^\top \boldsymbol{w}_\varrho
    + v_\varrho,
\end{align}
which are compatible with the policy by construction. The canonical
\texttt{SARSA} algorithm was used for policy evaluation with learning rate of
0.005. The policy updates were performed every \(N_\text{Policy} = 100\)
episodes with \(\eta = 0.001\).

\subsection{Portfolio Optimisation}
In the portfolio optimisation problem we again considered a Gibbs policy, but
this time of the form
\begin{equation}
    \pi_{\T}(a \mid \boldsymbol{s}) = \frac{e^{\boldsymbol{\phi}(\boldsymbol{s}, a)^\top \T}}{\sum_b
    e^{\boldsymbol{\phi}(\boldsymbol{s}, b)^\top \T}}.
\end{equation}
Here, we chose to use a linear basis over the state space, with independent
sets of activations for each action. That is,
\begin{equation}
    \boldsymbol{\phi}(\boldsymbol{s}, a) =
    \left[\boldsymbol{\phi}(\boldsymbol{s})^\top \circ \boldsymbol{1}^\top_{a\,
    =\, 1}, \ldots, \boldsymbol{\phi}(\boldsymbol{s})^\top \circ
\boldsymbol{1}^\top_{a\, =\, |\mathcal{A}|} \right]^\top
\end{equation}
where \(\circ\) denotes the Hadamard product and the state-dependent basis is
given by
\begin{equation}
    \boldsymbol{\phi}(\boldsymbol{s}) = \left[1, s_1, s_2, \ldots, s_N\right].
\end{equation}
The value functions were then represented by linear function approximators
\begin{align}
    \hat{q}(a) &= \nabla_{\T}\log\pi_{\T}(a \mid \boldsymbol{s})^\top
    \boldsymbol{w}_q + \boldsymbol{\phi}(\boldsymbol{s})^\top \boldsymbol{v}_q,
    \\
    \hat{\varrho}(a) &= \nabla_{\T}\log\pi_{\T}(a \mid \boldsymbol{s})^\top
    \boldsymbol{w}_\varrho + \boldsymbol{\phi}(\boldsymbol{s})^\top
    \boldsymbol{v}_\varrho,
\end{align}
which are compatible with the policy by construction. The canonical
\texttt{SARSA(\(\lambda\))} algorithm was used for policy evaluation with
learning rate of 0.0001, a discount factor of \(\gamma = 0.99\) and
accumulating trace with rate \(\lambda = 1\) (forgive the abuse of notation
wrt. the Lagrange multiplier). The policy updates were then performed every
\(N_\text{Policy} = 200\) time steps with \(\eta = 0.0001\). To improve
stability we pre-trained the value-function and Lagrange multiplier (learning
rate of 0.001) for 1000 episodes against the initial policy.

The portfolio optimisation domain itself was configured as follows: a liquid
interest rate $r_\text{L} = 1.005$, and illiquid interest rates
$\overline{r}_\text{N} = 1.25$ and $\underline{r}_\text{N} = 1.05$; switching
probabilities $p_\uparrow = 0.1$ and $p_\downarrow = 0.6$, and probability of
default $p_\text{D} = 0.1$; max order size of $M = 10$ with asset cost of
$\alpha = 0.2 / M$; maturity time of $N = 4$ steps and episode length of $50$
time steps.

\subsection{Optimal Consumption}
In the optimal consumption problem we had to deal with a 2-dimensional
continuous action space. For this we consider a policy with likelihood that is
the product of two independent probability distributions,
\begin{equation}
    \pi_{\T}(\boldsymbol{a} \mid \boldsymbol{s}) = \pi^{(1)}_{\T_1}(a_1 \mid
    \boldsymbol{s}) \cdot \pi^{(2)}_{\T_2}(a_2 \mid \boldsymbol{s}),
\end{equation}
where
\begin{align}
    \pi^{(1)}_{\T_1}(a_1 \mid \boldsymbol{s}) &=
    \frac{1}{\sqrt{2\pi\hat{\sigma}(\boldsymbol{s})^2}} e^{-\frac{a_1 -
    \hat{\mu}(\boldsymbol{s})}{2\hat{\sigma}(\boldsymbol{s})^2}}, \\
    \pi^{(2)}_{\T_2}(a_2 \mid \boldsymbol{s}) &=
    \frac{1}{B(\hat{\alpha}(\boldsymbol{s}), \hat{\beta}(\boldsymbol{s}))}
    a_2^{\hat{\alpha}(\boldsymbol{s}) - 1} (1 -
    a_2)^{\hat{\beta}(\boldsymbol{s}) - 1}.
\end{align}
In this case, \(\hat{\mu}\) was represented by a linear function approximator
with third-order Fourier basis, and \(\hat{\sigma}, \hat{\alpha}\) and
\(\hat{\beta}\) were given by the same as \(\hat{\mu}\) followed by a softplus
transformation to maintain positive values. Both \(\hat{\alpha}\) and
\(\hat{\beta}\) were also shifted by a value 1 to maintain unimodality.  The
value functions were then represented by linear function approximators
\begin{align}
    \hat{q}(a) &= \nabla_{\T}\log\pi_{\T}(a \mid \boldsymbol{s})^\top
    \boldsymbol{w}_q + \boldsymbol{\phi}(\boldsymbol{s})^\top \boldsymbol{v}_q,
    \\
    \hat{\varrho}(a) &= \nabla_{\T}\log\pi_{\T}(a \mid \boldsymbol{s})^\top
    \boldsymbol{w}_\varrho + \boldsymbol{\phi}(\boldsymbol{s})^\top
    \boldsymbol{v}_\varrho,
\end{align}
which are compatible with the policy by construction, and use the same Fourier
basis as the state-dependent \(\boldsymbol{\phi}(\boldsymbol{s})\). The
canonical \texttt{SARSA(\(\lambda\))} algorithm was used for policy evaluation
with learning rate of 0.00001, a discount factor of \(\gamma = 1\) and
accumulating trace with rate \(\lambda = 0.97\) (forgive the abuse of notation
wrt. the Lagrange multiplier). The policy updates were then performed every
\(N_\text{Policy} = 1000\) time steps with \(\eta = 0.00001\). To improve
stability we pre-trained the value-function and Lagrange multiplier (learning
rate of 0.0025) for 1000 episodes against the initial policy.

The optimal consumption domain itself was configured as follows: a drift of
$r_\text{L} = 0.05$ for the liquid asset; a risky asset whose price follows an
It\^{o} diffusion, $\text{Price}_{t+1} = \text{Price}_t + \mu_\text{R} +
\sigma_\text{R} \text{BM}_t$, where $\mu_\text{R} = 1$, $\sigma_\text{R} =
0.25$ and $\text{BM}_t$ denotes a standard Brownian motion; an initial wealth
of $W_0 = 1$ and time increment of $\Delta t = 0.005$; and a probability of
default at each time step of $p_\text{D} = 0.0015$.

\end{document}